\pdfoutput=1
\documentclass[conference]{IEEEtran}
\IEEEoverridecommandlockouts
\usepackage{cite}
\usepackage{amsmath,amssymb,amsfonts}
\usepackage{graphicx}
\usepackage{textcomp}
\usepackage{xcolor}

\usepackage{algorithm}
\usepackage{algpseudocode}

\usepackage{ntheorem}
\theoremseparator{.}
\theoremstyle{plain}
\newtheorem{remark}{Remark}
\newtheorem{theorem}{Theorem}
\newtheorem{lemma}{Lemma}
\newtheorem{prop}{Proposition}
\newtheorem{defi}{Definition}

\newtheorem{assu}{Assumption}
\newtheorem*{proof}{Proof}

\usepackage{subfigure}

\usepackage{mathtools}
\mathtoolsset{showonlyrefs}

\usepackage{cases}

\def\BibTeX{{\rm B\kern-.05em{\sc i\kern-.025em b}\kern-.08em
    T\kern-.1667em\lower.7ex\hbox{E}\kern-.125emX}}
\begin{document}

\title{Data-Heterogeneous Hierarchical Federated Learning with Mobility\\
}

\author{Tan Chen$^*$, Jintao Yan$^*$, Yuxuan Sun$^\dag$, Sheng Zhou$^*$, Deniz G\"und\"uz$^\ddag$, Zhisheng Niu$^*$\\
$^*$Beijing National Research Center for Information Science and Technology\\
Department of Electronic Engineering, Tsinghua University, Beijing 100084, China\\
$^\dag$School of Electronic and Information Engineering, Beijing Jiaotong University, Beijing 100044, China\\
$^\ddag$Department of Electrical and Electronic Engineering, Imperial College London, London SW7 2BT, UK\\
Email:\{chent21,yanjt22\}@mails.tsinghua.edu.cn, yxsun@bjtu.edu.cn,\\ sheng.zhou@tsinghua.edu.cn, d.gunduz@imperial.ac.uk, niuzhs@tsinghua.edu.cn}


\maketitle

\begin{abstract}
Federated learning enables distributed training of machine learning (ML) models across multiple devices in a privacy-preserving manner. Hierarchical federated learning (HFL) is further proposed to meet the requirements of both latency and coverage. 
In this paper, we consider a data-heterogeneous HFL scenario with mobility, mainly targeting vehicular networks. We derive the convergence upper bound of HFL with respect to mobility and data heterogeneity, and analyze how mobility impacts the performance of HFL. While mobility is considered as a challenge from a communication point of view, our goal here is to exploit mobility to improve the learning performance by mitigating data heterogeneity. Simulation results verify the analysis and show that mobility can indeed improve the model accuracy by up to 15.1\% when training a convolutional neural network on the CIFAR-10 dataset using HFL. 
\end{abstract}


\section{Introduction}

The advent of 5G has revolutionized intelligent vehicles, enabling them to generate and share significant data volumes through vehicle-to-everything (V2X) services\cite{elbir2022federated}. In this context, machine learning (ML) becomes an essential tool thanks to its ability to efficiently analyze vast amounts of data while adapting to the dynamics of the mobile environment\cite{sun2022meet,tan2020federated,liang2018toward}. Conventional ML solutions rely on offloading edge data to cloud servers. Such centralized solutions, however, encounter challenges of limited communication resources and privacy concerns in vehicular networks \cite{ye2020federated,posner2021federated,zhou2023toward}. As an alternative, federated learning (FL) has gained popularity for its ability to efficiently utilize communication resources while preserving privacy\cite{mcmahan2017communication}. Furthermore, to surmount the unstable communication links in cloud-based FL and the limited coverage and vehicle density in edge-based FL\cite{elbir2022federated}, hierarchical FL (HFL) is further proposed in \cite{liu2020client} and \cite{abad2020hierarchical}. The goal of HFL is to train a global model of a central cloud server in a federated manner with the help of edge servers. The edge servers are closer to the devices, and aggregate model parameters of the devices in their coverage area, while the global aggregation at the cloud server takes place less frequently, resulting in a trade-off between the training time and the covering vehicles.

In this paper, we consider HFL in a vehicular network (see Fig. \ref{HFL} for an illustration). Implementing HFL in vehicular networks faces two challenges. Firstly, different vehicles have different routes and may collect data with different statistics (e.g., different classes), resulting in the production of heterogeneous (or non-i.i.d.) data\cite{ayache2023walk,mcmahan2017communication}. Heterogeneous data causes the local objective function to diverge from the global objective function, thereby degrading the learning performance\cite{shi2020joint,wang2021addressing}.
By minimizing the Kullback-Leibler divergence among the data distributions across the edge servers, a user-edge association method is proposed in \cite{deng2021share} to reduce the total number of communication rounds. A user-edge association problem is presented in \cite{liu2022joint}, aiming to minimize the parameter difference between the global and optimal models. The connection between the parameter difference and the data distribution difference is then established. 

Secondly, unlike in traditional HFL, where the server and clients are fixed, the mobility of vehicles causes the network topology to be constantly changing. The impact of mobility in HFL is examined in \cite{feng2022mobility}. The convergence speed in relation to mobility rate is analyzed, showing that mobility decreases the convergence speed of HFL. An algorithm is then proposed to alleviate the impact of mobility by aggregating local models based on cosine similarity among model parameters. Experiments show the proposed algorithm improves the performance of HFL with heterogeneous data and mobility. However, the authors
view mobility as merely a negative factor in training. 

In reality, mobility has two-fold effects on HFL. On the one hand, the variations in channel quality make it difficult to estimate and adapt to channels for reliable communication. On the other hand, mobility also promotes the mixing of heterogeneous data, which can potentially improve the learning performance\cite{sun2022meet}.

In this work, we investigate the effect of mobility on the performance of HFL. The main contributions of this work are as follows.

\begin{itemize}

\item We analyze the convergence upper bound of HFL with respect to data 
 heterogeneity and mobility, showing that mobility  influences the data heterogeneity, and thus, the convergence upper bound.
 
\item Through numerical experiments on data-heterogeneous HFL with mobility, we show that mobility can indeed \textit{enhance} the convergence speed and accuracy of HFL. 

\end{itemize}

\begin{figure}[ht!]
\centering
\includegraphics[width=0.46\textwidth]{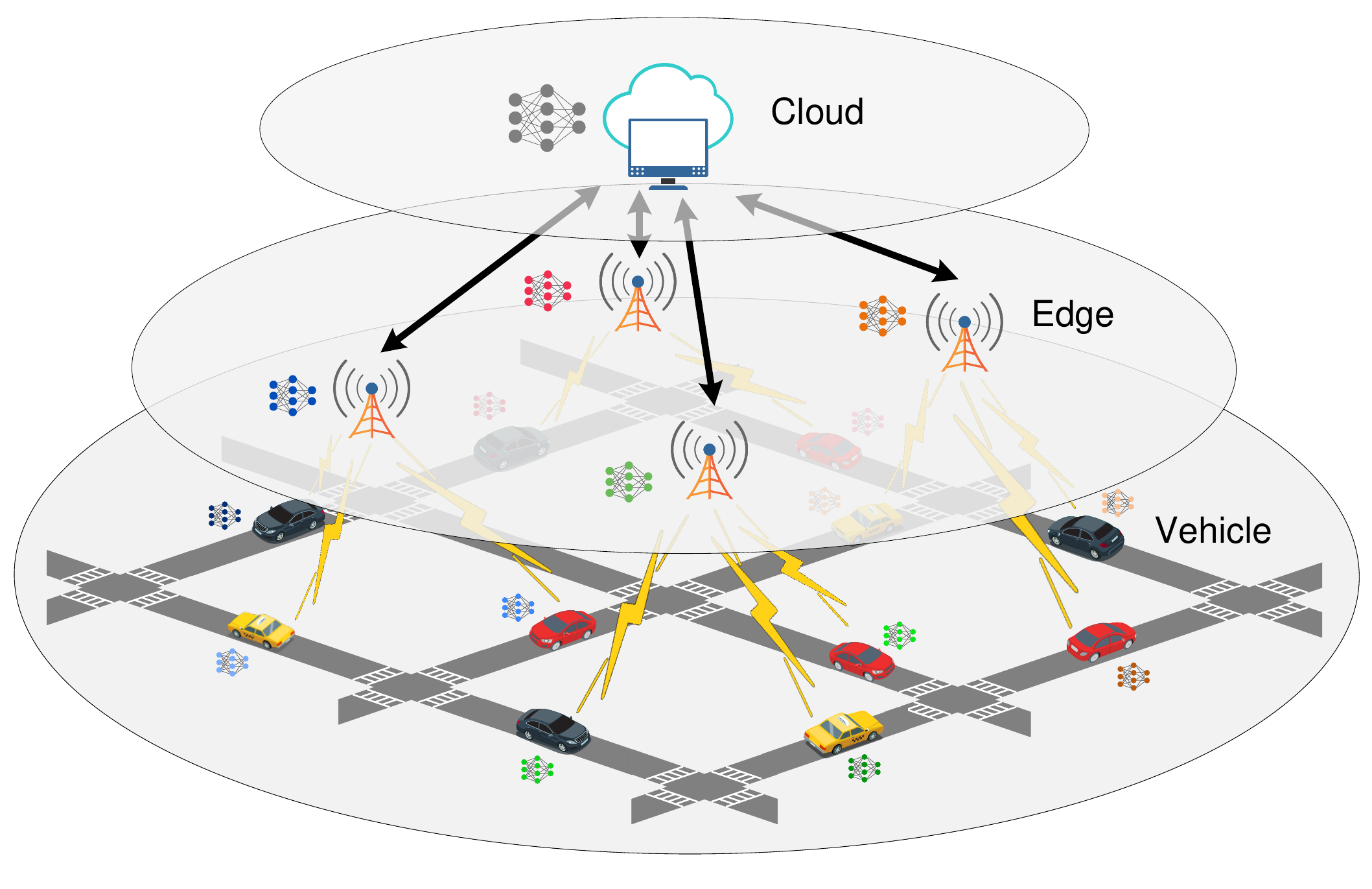}
\caption{Illustration of HFL in vehicular networks.}
\label{HFL}
\end{figure}

The rest of this paper is organized as follows. In Section 
\ref{Sec-3}, we describe the system model, characterize the learning task, and present the training algorithm. In Section \ref{Sec-4}, a convergence analysis is conducted, and a convergence bound is derived for data-heterogeneous HFL with mobility. Section \ref{Sec-5} presents simulation results and discussions. Finally, Section \ref{Sec-6} concludes the work.
\section{HFL SYSTEMS}
\label{Sec-3}


We consider HFL in vehicular networks. A central cloud server controls several edge servers, which can represent base stations or roadside units. Each edge server is static, and covers a limited area of streets used by the vehicles. The vehicles move on the streets stochastically and occasionally cross the coverage of an edge server. We assume that the cloud server has a large enough coverage, so all the vehicles are always within its coverage. Denote the cloud server as $c$, and assume that there are $N$ edge servers and $M$ vehicles. In practice, $M$ is much larger than $N$.  The vehicles in the system want to cooperatively train a model with the help of the cloud and edge servers.



The HFL task attempts to find the connection between inputs $\boldsymbol{x}_i$ and labels $y_i$ in the global dataset $\mathcal{D}=\{\boldsymbol{x}_i,y_i \}_{i=1}^{|\mathcal{D}|}$. Let $\mathcal{D}_m$ denote the dataset of the $m$-th vehicle; we have $\mathcal{D}=\cup_{m=1}^M \mathcal{D}_m$. 

For a sample $\{\boldsymbol{x}_i,y_i \}$, let $g_i \left(\boldsymbol{w}\right)$ be the sample loss function for model parameters $\boldsymbol{w}$. Then the loss function of the $m$-th vehicle is given by $f_m\left(\boldsymbol{w}\right)=\frac1{|\mathcal{D}_m|} \sum_{i\in \mathcal{D}_m}g_i \left(\boldsymbol{w}\right)$, and the global loss function is then determined by an average of the sample loss functions, $F\left(\boldsymbol{w}\right)=\frac1{|\mathcal{D}|} \sum_{i\in \mathcal{D}}g_i \left(\boldsymbol{w}\right)= \sum_{m=1}^M \frac{|\mathcal{D}_m|}{|\mathcal{D}|} f_m\left(\boldsymbol{w}\right) $. The training objective  of HFL is $\min_{\boldsymbol{w}} F\left(\boldsymbol{w}\right)$.

\subsection{Mob-HierFAVG}
The HierFAVG algorithm in \cite{liu2020client} provides a solution to the HFL problem. We build our solution upon HierFAVG by adding the mobility factor. 

We denote the gradient descent process on a batch of samples at each vehicle as a local update. We maintain an iteration enumerator $\tau$ shared by all nodes in the system, recording the number of local updates each vehicle has carried out in total. A synchronized system is assumed, so $\tau$ is tracked by all the vehicles. We denote the model parameters of the $m$-th vehicle, the $n$-th edge server, and the cloud server at iteration $\tau$ as $\boldsymbol{w}_{m}^{\left(\tau\right)}, \boldsymbol{w}_{e,n}^{\left(\tau\right)}, \boldsymbol{w}^{\left(\tau\right)}$, respectively.


We assume that the vehicles sample data before training, and the dataset carried by each vehicle do not change during training. All the participants initialize a global model $\boldsymbol{w}^{[0]}$, and each vehicle repeats updating the model with its local data.

For each local update, vehicle $m$ performs stochastic gradient decent using its local data by ${\boldsymbol{w}}_{m}^{\left(\tau\right)}= \boldsymbol{w}_{m}^{({\tau}-1)}-\eta \nabla f_{m} \left(\boldsymbol{w}_{m}^{\left(\tau\right)},\boldsymbol{\xi}_{m}^{\left(\tau\right)}\right)$, where $\eta$ is the learning rate and $\nabla f_{m} \left(\boldsymbol{w}_{m}^{\left(\tau\right)},\boldsymbol{\xi}_{m}^{\left(\tau\right)}\right)$ is the stochastic gradient of the loss function generated by a batch $\boldsymbol{\xi}_{m}^{\left(\tau\right)}$ sampled from the local dataset.

The edge servers aggregate local models in their coverage area every $\tau_l$ local updates. During an edge aggregation, each vehicle uploads its model parameters to the edge server it is associated with at that moment. The edge server then aggregates all the models in its coverage and distributes the aggregated model back to the vehicles. We assume that each vehicle is associated with only the nearest edge server, and each edge aggregates the models of all the vehicles within its coverage.
The cloud server aggregates edge models every $\tau_e$ edge aggregations. 

Let $\mathcal{E}_n^{\left(\tau\right)}$ represent the vehicle set connected to the $n$-th edge server at iteration $\tau$. The edge aggregation can be expressed by ${\boldsymbol{w}}_{e,n}^{\left(\tau\right)}=\sum_{m\in \mathcal{E}_n^{\left(\tau\right)}}\alpha^{(\tau)}_{m,n} {\boldsymbol{w}}_{m}^{\left(\tau\right)}$, where $\alpha^{(\tau)}_{m,n}\!\triangleq\!\frac{|\mathcal{D}_{m}|}{\sum_{m'\in \mathcal{E}_n^{\left(\tau\right)}}|\mathcal{D}_{m'}|}$. Similarly, the cloud aggregation is expressed as $\boldsymbol{w}^{\left(\tau\right)} = \sum_n \theta^{(\tau)}_n {\boldsymbol{w}}_{e,n}^{\left(\tau\right)} $, where $ \theta^{(\tau)}_{n}\!\triangleq\!\frac{\sum_{m\in \mathcal{E}_n^{\left(\tau\right)}}|\mathcal{D}_{m}|}{\sum_{m=1}^M|\mathcal{D}_{m}|}$. 
The evolution of the $n$-th edge model is denoted by
\begin{equation}
\boldsymbol{w}_{e,n}^{\left(\tau\right)}=\begin{cases} 
        \Tilde{\boldsymbol{w}}_{e,n}^{\left(\tau\right)}=\sum_{m\in \mathcal{E}_n^{\left(\tau\right)}}\alpha_{m,n}^{\left(\tau\right)} \Tilde{\boldsymbol{w}}_{m}^{\left(\tau\right)},\quad \tau_l\mid\tau, \tau_l\tau_e\nmid\tau \\ 
        \sum_n \theta_n^{\left(\tau\right)} \Tilde{\boldsymbol{w}}_{e,n}^{\left(\tau\right)},\quad \tau_l\tau_e\mid\tau \end{cases}\hspace{-9pt}\!,\! 
\end{equation}
and the evolution of the $m$-th local model is denoted by
\begin{equation}
        \boldsymbol{w}_{m}^{\left(\tau\right)}=\begin{cases} 
        \Tilde{\boldsymbol{w}}_{m}^{\left(\tau\right)}=\boldsymbol{w}_{m}^{\left({\tau}-1\right)}-\eta\nabla f_{m} \left(\boldsymbol{w}_{m}^{\left({\tau}-1\right)}\right),\quad \tau_l\nmid\tau \\ 
        \boldsymbol{w}_{e,n}^{\left(\tau\right)},\quad \tau_l\mid\tau \end{cases}\hspace{-9pt}\!.\!
\end{equation}
Here $a\mid b$ means $b$ is 
divisible by $a$, while $a\nmid b$ means $b$ is not divisible by $a$.

\subsection{Data Distribution}
Since the data is sampled before training, the initial data distribution influences the distribution of training data. We mainly consider these three typical scenarios:

(1) \emph{i.i.d.} The learning tasks are homogeneous for all vehicles, and each vehicle owns a dataset of the same distribution.

(2) \emph{local non-i.i.d.} The learning tasks are heterogeneous for vehicles, such as trajectory prediction and edge caching. For these tasks, vehicles own datasets of different distributions.

(3) \emph{edge non-i.i.d.} The learning tasks are relevant to the environment and thus have a spatial correlation, such as intersection management, beam selection, and channel estimation. In the third scenario, the data distributions are non-i.i.d. across edges, but i.i.d. across the vehicles associated with the same edge server at the time of data collection.

\section{Convergence Analysis}
\label{Sec-4}
\subsection{Definitions}
We quantify data heterogeneity by the gradient difference.
    \begin{defi}
Define $\delta_m$ and $\Delta_n^{(\tau)}$ as
\begin{align}
    {\left\lVert {\nabla f_m\left(\boldsymbol{w}\right)-\nabla F\left(\boldsymbol{w}\right)} \right\rVert}
    &\le \delta_m, \\
    {\left\lVert {\nabla F_{e,n}^{(\tau)}\left(\boldsymbol{w}\right)-\nabla F\left(\boldsymbol{w}\right)} \right\rVert }
    &\le \Delta_n^{(\tau)},
\end{align}
\end{defi}
where $F_{e,n}^{(\tau)}\left(\boldsymbol{w}\right)=\sum_{m\in \mathcal{E}_n^{(\tau)}}\alpha_{m,n}^{(\tau)}f_{m}\left(\boldsymbol{w}\right).$ Here, $\delta_m$ and $\Delta_n^{(\tau)}$ represent the upper bound on the local gradient difference between the $m$-th vehicle and the cloud server, and the upper bound on the edge gradient difference between the $n$-th edge server and the cloud server at the $\tau$-th local iteration, respectively.

Furthermore, the edge-level local gradient difference is represented by
$
    \delta_n^{(\tau)}=\sum_{m\in \mathcal{E}_n^{(\tau)}} \alpha_{m,n}^{(\tau)}\delta_m,
$
and the cloud-level local gradient difference and edge gradient difference by 
$
    \delta=\sum_m \alpha_m\delta_m$ and $
    \Delta^{(\tau)}=\sum_n \theta_n^{(\tau)}\Delta_n^{(\tau)}
$, respectively, where $\alpha_m\!\triangleq\!\frac{|\mathcal{D}_m|}{|\mathcal{D}|}$.

\begin{remark}

The definition indicates that the local gradient difference is not time-varying, while the edge gradient difference is time-varying. This is because the local datasets remain unchanged during training, but the edge datasets are changing due to the mobility of vehicles.\vspace{-5pt}

\end{remark}

From the definition, $\delta_m$ and $\Delta_n^{(\tau)}$ can represent the data heterogeneity and mobility of vehicles. Furthermore, to prove the convergence of HFL, we define the following parameters:
\begin{itemize}
    \item $\boldsymbol{u}^{(\tau)}$: The virtual cloud parameter that records the weighted sum of local parameters at each iteration:
    \begin{equation}
        \boldsymbol{u}^{(\tau)}=\sum_{m}\alpha_{m} \boldsymbol{w}_{m}^{(\tau)}.
    \end{equation}\vspace{-8pt}
    
    \item $\boldsymbol{v}^{(\tau)}$: The virtual centralized parameter that synchronizes with the virtual cloud parameter periodically. It evolves as\vspace{-3pt}
    \begin{equation}
        \boldsymbol{v}^{(\tau)}=\begin{cases} 
        \Tilde{\boldsymbol{v}}^{(\tau)}=\boldsymbol{v}^{(\tau-1)}-\eta\nabla F\left(\boldsymbol{v}^{(\tau-1)}\right),\quad \tau_l\tau_e\nmid\tau \\ 
        \boldsymbol{u}^{(\tau)},\quad \tau_l\tau_e\mid\tau \end{cases}\hspace{-9pt}.
    \end{equation}\vspace{-10pt}
    
\end{itemize}


\subsection{Convergence}
We first introduce a general convergence bound for HFL.

\begin{assu}\label{assu1} We assume the following for all $m$:

\begin{enumerate}
\item $f_m\left(\boldsymbol{w}\right)$ is convex;
\item $f_m\left(\boldsymbol{w}\right)$ is $\rho$-Lipschitz, i.e., 

$\left\lVert {f_m\left(\boldsymbol{w}\right)-f_m\left(\boldsymbol{w}' \right)} \right\rVert\le \rho \left\lVert {\boldsymbol{w}-\boldsymbol{w}'}  \right\rVert$ for any $\boldsymbol{w},\boldsymbol{w}'$;

\item $f_m\left(\boldsymbol{w}\right)$ is $\beta$-smooth, i.e., 

$\left\lVert {\nabla f_m\left(\boldsymbol{w}\right)-\nabla f_m\left(\boldsymbol{w}'\right)} \right\rVert\le \beta \left\lVert {\boldsymbol{w}-\boldsymbol{w}'}  \right\rVert$ for any $\boldsymbol{w},\boldsymbol{w}'$.
\end{enumerate}

\end{assu}

\begin{prop}\label{prop1}
    For any $m$, assume $f_m\left(\boldsymbol{w}\right)$ is $\rho$-Lipschitz,$\beta$-smooth and convex, and denote the optimal model parameters as $\boldsymbol{w}^*$. Assume for some $\epsilon\ge 0$, 
    we have
    
    (1) $\eta\le\frac1{\beta}$,
    (2) $\eta\varphi-\frac{\rho U_k}{\tau_l\tau_e\epsilon^2}>0$ for all $k$,
    
    (3)$F\left(\Tilde{\boldsymbol{v}}^{\left(k\tau_l\tau_e\right)}\right)-F\left(\boldsymbol{w}^*\right)\ge\epsilon$,\label{cond}
    
    (4)$F\left(\boldsymbol{w}^{\left(k\tau_l\tau_e\right)}\right)\ge\epsilon$ for all $k$, where we define 
    \begin{equation}
    \varphi=\min\limits_k\frac{1-\frac{\beta\eta}{2}}{\left\lVert{\Tilde{\boldsymbol{v}}^{\left(\left(k-1\right)\tau_l\tau_e\right)}-\boldsymbol{w}^*}\right\rVert^2},\quad
    {\left\lVert{\boldsymbol{u}^{\left(k\tau_l\tau_e\right)}-\Tilde{\boldsymbol{v}}^{\left(k\tau_l\tau_e\right)}}\right\rVert}\le U_k.
    \end{equation}\vspace{-13pt}
    \\
    Here, $U_k$ represents an upper bound on the central-cloud difference ${\big\lVert{\boldsymbol{u}^{\left(k\tau_l\tau_e\right)}-\Tilde{\boldsymbol{v}}^{\left(k\tau_l\tau_e\right)}}\big\rVert}$ at the $k$-th cloud epoch. Then after $T\!\triangleq\!K\tau_l\tau_e $ local updates, we have the following convergence upper bound for HFL\vspace{-2pt}
    \begin{equation}
        F\left(\boldsymbol{w}^{\left(T\right)}\right)-F\left(\boldsymbol{w}^*\right)\le\frac1{T\eta\varphi-\frac{\rho\sum\limits_{k=1}^KU_k}{\epsilon^2}}.
    \end{equation}\label{prop1}\vspace{-10pt}
\end{prop}
\begin{proof}
    This is an extension of Lemma 2 in \cite{wang2019adaptive}.
\end{proof}\vspace{-3pt}

    Proposition \ref{prop1} indicates that for a fixed cloud epoch $K$, learning rate $\eta$, and aggregation periods $\tau_l, \tau_e$, the convergence bound is determined by $U_k$. Therefore, we then focus on bounding $U_k$.\vspace{-3pt}

\begin{lemma}\label{lemma1}
    We have the following for the virtual parameters:\vspace{-2pt}
    \begin{flalign}
        &\left\lVert{\boldsymbol{u}^{(\tau)}\!-\!\boldsymbol{\Tilde{v}}^{(\tau)}}\right\rVert\le&
    \end{flalign}\vspace{-15pt}
    \begin{align}
    \begin{cases}
    \left\lVert{\boldsymbol{u}^{(\tau-1)}\!-\!\boldsymbol{v}^{(\tau-1)}}\right\rVert
    \!+\!\eta\beta \sum_m\alpha_m \left\lVert{\boldsymbol{w}_m^{(\tau-1)}\!-\!\boldsymbol{v}^{(\tau-1)}}\right\rVert,\\
    \hspace{13em}\tau_l\nmid\tau\!-\!1 \\
    \left\lVert{\boldsymbol{u}^{(\tau-1)}\!-\!\boldsymbol{v}^{(\tau-1)}}\right\rVert
    \!+\!\eta\beta\hspace{-1pt} \sum_n\!\theta_n^{(\tau-1)} \!\left\lVert{{\boldsymbol{u}}_{n}^{(\tau-1)}\!-\!\boldsymbol{v}^{(\tau-1)}}\right\rVert,\\
    \hspace{13em}\tau_l\mid\tau\!-\!1,\tau_l\tau_e\nmid\tau\!-\!1 \\
    0, \hspace{12.05em}\tau_l\tau_e\mid\tau\!-\!1,
    \end{cases}\label{mobclouddiff}
    \end{align}\vspace{-6pt}


    where ${\boldsymbol{u}}^{(\tau)}_n$ denotes the virtual edge parameter satisfying
    \begin{equation}
        {\boldsymbol{u}}^{(\tau)}_n=\sum_{m\in \mathcal{E}_n^{(\tau)}}\alpha^{(\tau)}_{m,n} \boldsymbol{w}_{m}^{(\tau)}.
    \end{equation}\vspace{-14pt}
\end{lemma}
\begin{proof}
    For the sake of simplicity, we assume $\boldsymbol{\xi}_{m}^{\left(\tau\right)}=\mathcal{D}_m$, since stochastic gradient descent has been proven to be an approximation to deterministic gradient descent\cite{tuor2018distributed}. So we substitute $\nabla f_{m} \left(\boldsymbol{w}_{m}^{\left(\tau\right)},\boldsymbol{\xi}_{m}^{\left(\tau\right)}\right)$ by $\nabla f_{m} \left(\boldsymbol{w}_{m}^{\left(\tau\right)}\right)$, and obtain
\vspace{-5pt}
    \begin{align}
    &\hspace{1.15em}\left\lVert \boldsymbol{u}^{(\tau)}-\Tilde{\boldsymbol{v}}^{(\tau)}\right\rVert  \vspace{2ex}\\
    &=\bigg\lVert \!\Big[
        \boldsymbol{u}^{(\tau-1)}\!-\!\eta\sum\limits_m\alpha_m\nabla f_m\left(\boldsymbol{w}_m^{(\tau-1)}\right)\Big]
        \\[-3pt]&\hspace{1.5em}
    \!-\!\left[
        \boldsymbol{v}^{(\tau-1)}\!-\!\eta\nabla F\left(\boldsymbol{v}^{(\tau-1)}\right)\right]\bigg\rVert\\[-4pt]
    &=\bigg\lVert \Big[
        \boldsymbol{u}^{(\tau-1)}\!-\!\boldsymbol{v}^{(\tau-1)}\Big]\\[-4pt]
    &\hspace{1.5em}\!-\!\eta\sum\limits_m
    \alpha_m\Big[\nabla f_m\left( \boldsymbol{w}_m^{(\tau-1)}\right)\!-\!\nabla f_m\left(\boldsymbol{v}^{(\tau-1)}\right)\Big] \bigg\rVert.\label{proof1}
    \end{align}
    When $\tau_l\nmid\tau\!-\!1$, we derive the first formula in \eqref{mobclouddiff} by the triangle inequality and $\beta$-smoothness. When $\tau_l\mid\tau\!-\!1$ and $\tau_l\tau_e\nmid\tau\!-\!1$, we have $\boldsymbol{w}_m^{(\tau-1)} = \boldsymbol{w}_{m'}^{(\tau-1)}$ if $m, m' \in\mathcal{E}_n^{(\tau-1)}$, and thus
    \begin{align}
        \eqref{proof1}=&\bigg\lVert{\left[
    \boldsymbol{u}^{(\tau-1)}\!-\!\boldsymbol{v}^{(\tau-1)}\right]}\\
    &{-\eta \sum\limits_n
    \theta_n^{(\tau-1)}\left[\nabla F_n\left( \boldsymbol{u}_n^{(\tau-1)}\right)\!-\!\nabla F_n\left(\boldsymbol{v}^{(\tau-1)}\right)\right]} \bigg\rVert,
    \end{align}
    so the second formula in \eqref{mobclouddiff} can be similarly derived.
    When $\tau_l\tau_e\mid\tau\!-\!1$, we have
    \begin{align}
        \eqref{proof1}
        &=    \left\lVert {\big[\boldsymbol{u}^{\left(\tau-1\right)}\!-\!\boldsymbol{v}^{\left(\tau-1\right)}\big]
        \!-\!\eta\big[\nabla F\left(\boldsymbol{u}^{\left(\tau-1\right)}\right)\!-\!\nabla F\left(\boldsymbol{v}^{\left(\tau-1\right)}\right)\big]} \right\rVert\\[-8pt]
        &=0,
    \end{align}
    where the second equality holds because $\boldsymbol{u}^{\left(\tau-1\right)}=\boldsymbol{v}^{\left(\tau-1\right)}$ when $\tau_l\tau_e\mid\tau\!-\!1$ according to the definition.
    \end{proof}
\begin{lemma}
    Let $\tau=k\tau_l\tau_e\!+\!\tau_0$ for some $\tau_0\in(0,\tau_l\tau_e]$. We have
\begin{align}
    &\left\lVert {\boldsymbol{w}_m^{(\tau)}\!-\!\boldsymbol{\Tilde{v}}^{(\tau)}} \right\rVert \le
    \frac{\delta_m}{\beta}[\left(1+\eta\beta\right)^{\tau_0}\!-\!1].
    \label{localdiff}
\end{align}
    
\end{lemma}

\begin{proof}
    This is follows from Lemma 3 in \cite{wang2019adaptive}.
\end{proof}

\begin{lemma}\label{lemma3}
    Let $\tau=k\tau_l\tau_e\!+\!\tau_0$ for some $\tau_0\in(0,\tau_l\tau_e]$. We have
    \begin{align}
        \left\lVert {{\boldsymbol{u}}_{n}^{(\tau)}\!-\!\boldsymbol{\Tilde{v}}^{(\tau)}} \right\rVert 
        \!\le\! \frac{\delta_n^{(\tau)}}{\beta}[\left(1+\eta\beta\right)^{\tau_0}\!-\!1]\!-\!\eta\tau_0\!\left(\!\delta_n^{(\tau)}\!-\!\Delta_n^{(\tau)}\!\right)\!.\label{mobedgediff}
\end{align}
\end{lemma}
\begin{proof}
    Similarly to the proof of Lemma \ref{lemma1}, we can write the virtual edge parameter in the recursive form:
    \begin{align}
        \left\lVert {\boldsymbol{u}_n^{\left(\tau\right)}\!-\!\Tilde{\boldsymbol{v}}^{\left(\tau\right)}} \right\rVert 
        \le&\left\lVert {\boldsymbol{u}_n^{\left(\tau-1\right)}\!-\!\boldsymbol{v}^{\left(\tau-1\right)}} \right\rVert \\ 
         &\!+\!\eta\beta\hspace{-4pt}\sum\limits_{m\in \mathcal{E}_n}\hspace{-4pt}\alpha_{m,n}^{(\tau-1)}\hspace{-2pt}
        \left\lVert  {\boldsymbol{w}_m^{\left(\tau-1\right)}\!-\!\boldsymbol{v}^{\left(\tau-1\right)}\hspace{-1pt}}\right\rVert\!+\!\eta\Delta_n^{(\tau-1)}\hspace{-1pt}.
    \end{align}\vspace{-17pt}
    
    Since $\boldsymbol{v}^{\left(\tau-1\right)}=\boldsymbol{\tilde{v}}^{\left(\tau-1\right)}$ when $\tau_l\tau_e\nmid\tau$, we substitute $\left\lVert\boldsymbol{w}_m^{\left(\tau-1\right)}-\boldsymbol{v}^{\left(\tau-1\right)}\right\rVert$ by \eqref{localdiff}. Then we can prove Lemma \ref{lemma3} by mathematical induction.
\end{proof}

\begin{theorem}\label{theo1}
 For data-heterogeneous HFL with mobility, the central-cloud difference $\big\lVert \boldsymbol{u}^{(k\tau_l\tau_e)}\!-\!\Tilde{\boldsymbol{v}}^{(k\tau_l\tau_e)}\big\rVert$ has an upper bound
\begin{flalign}
 &U_{k}=
r(\tau_l\tau_e, \eta, \delta)\!-\!\eta\tau_l[\frac12\tau_e\left(\tau_e\!-\!1\right)\delta\!-\!\sum\limits_{j=1}^{\tau_e-1}j\Delta^{[k\tau_e+j]}],&\label{theo1eq}
\end{flalign}
where $r(\tau, \eta, \delta)=
        \frac{\delta}{\beta}\big[\left(1+\eta\beta\right)^{\tau}-1\big]-\tau\eta\delta$, and $\Delta^{[j]}=\Delta^{(j\tau_l)}$.
\end{theorem}

\begin{proof}    
The proof follows by substituting \eqref{localdiff} and \eqref{mobedgediff} into \eqref{mobclouddiff}, and summing up from $\tau=k\tau_l\tau_e\!+\!1$ to $\tau=(k\!+\!1)\tau_l\tau_e$.
\end{proof}

    This theorem indicates that mobility impacts the convergence upper bound of HFL by changing the edge gradient difference $\Delta^{[j]}$. If $\Delta^{[j]}=\Delta^{[0]}$ for all $j$, then the theorem degenerates into the case without mobility.  

    Specifically, considering the data distribution mentioned in Section \ref{Sec-3}, we have $\Delta^{[j]}\approx0$ for all $j$ in the i.i.d. case, since the mobility of vehicles does not impact the statistics of the datasets within the coverage area of an edge server. However, for the edge non-i.i.d. case, we have $\Delta^{[0]}\gg0$, since the initial datasets of each edge group is different. As vehicles move during training, the data across edge groups gradually mix up, so $\Delta^{[j]}$ may decrease over iterations. From \eqref{prop1} and \eqref{theo1eq}, $\Delta^{[j]}$ is positively correlated with the convergence upper bound, so mobility may increase the convergence speed in the edge non-i.i.d. case by decreasing the degree of data heterogeneity.
\section{Simulation Results}
\label{Sec-5}


\begin{figure*}[!t]
	\centering
    \subfigure[Comparison of the max test accuracy within 600 cloud epochs for different initial data distribution.]{\label{fig2}	
		\includegraphics[width=0.3\textwidth]{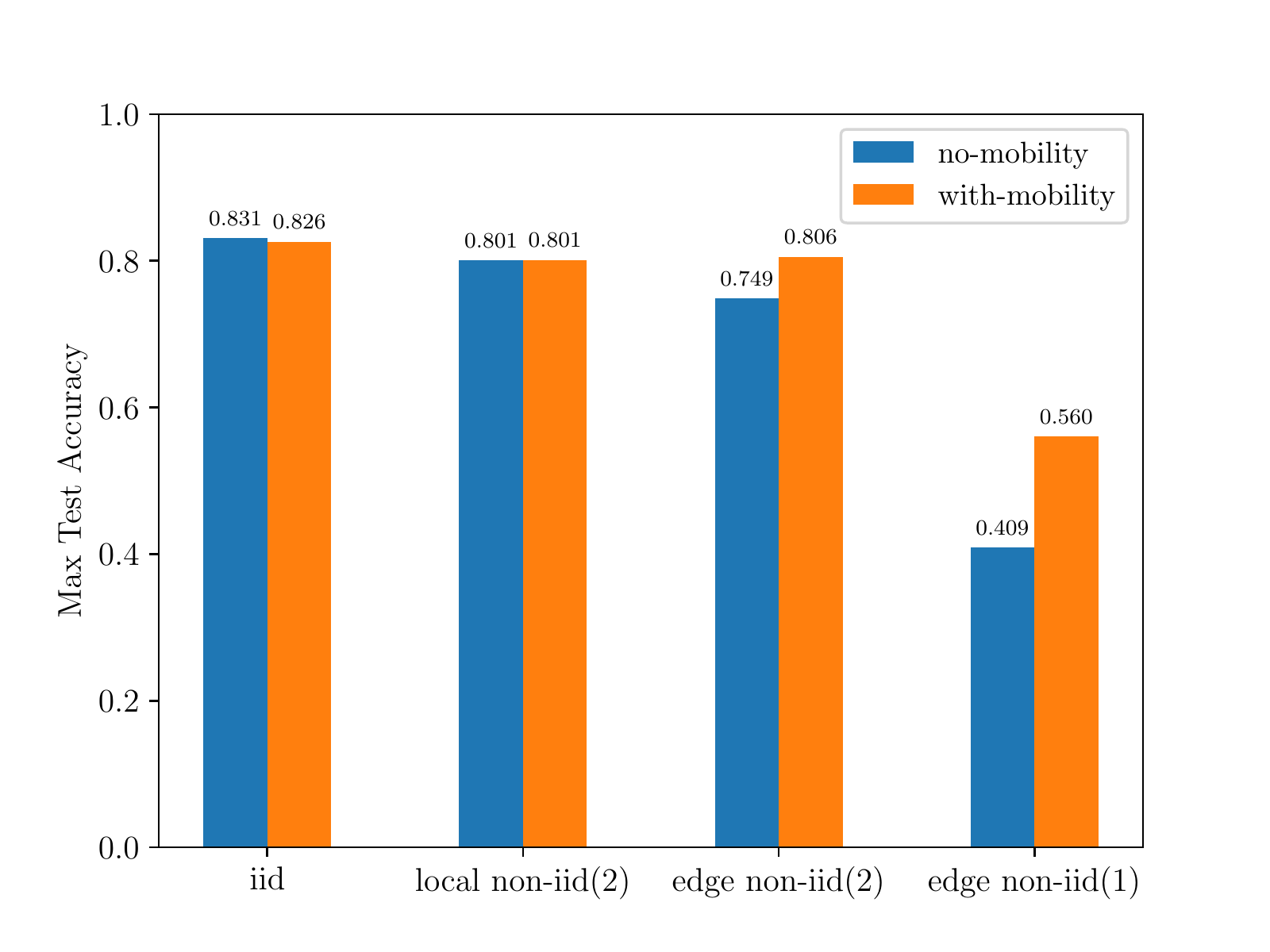} }
	\subfigure[Accuracy on the test dataset for the edge non-i.i.d.(2) case with different vehicle speeds, resuming from the model of 60\% test accuracy.]{\label{fig5}			
		\includegraphics[width=0.3\textwidth]{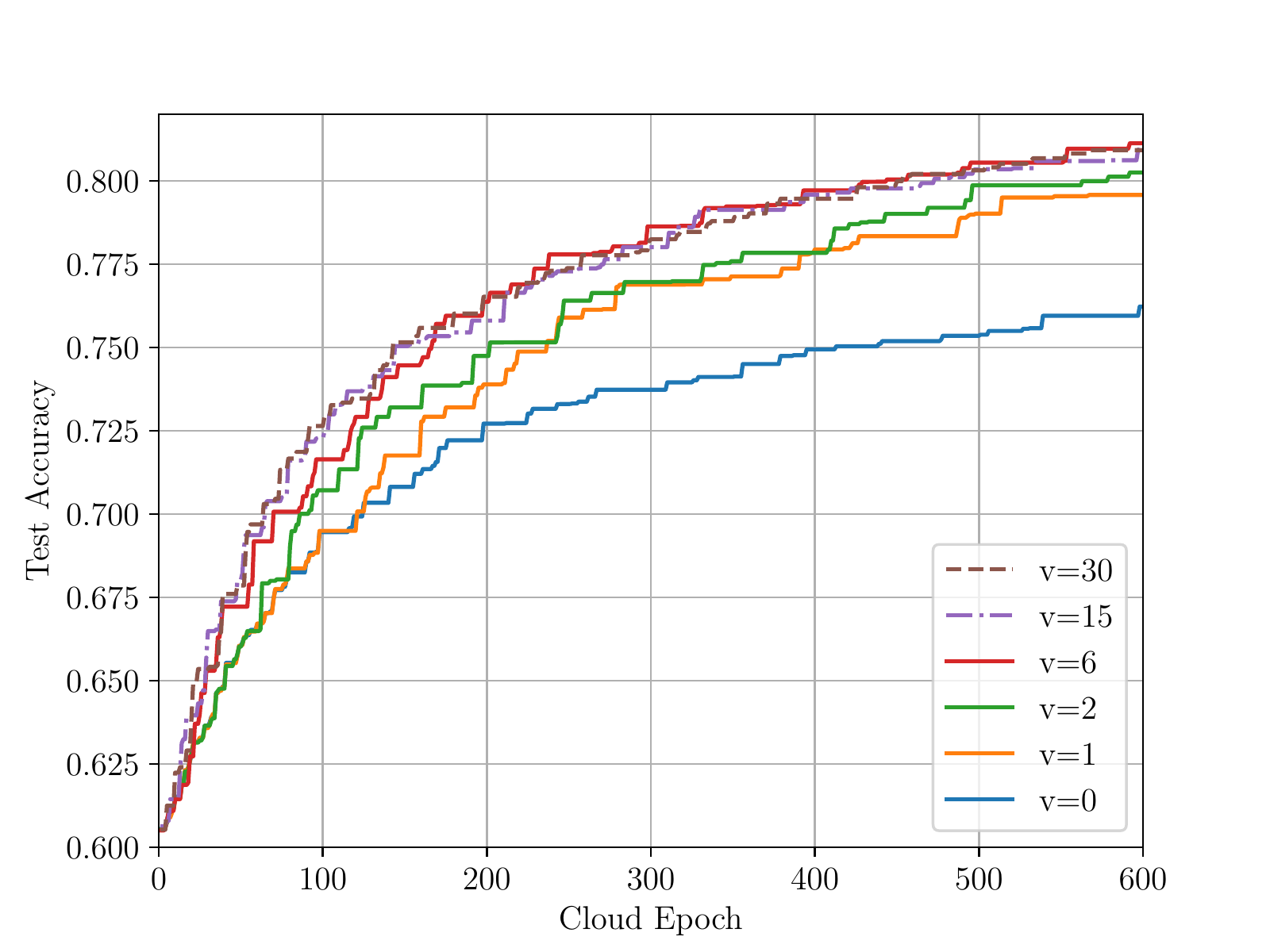} }		
	\subfigure[Rounds to reach the target accuracies for the edge non-i.i.d.(2) case, resuming from the model of 60\% test accuracy.]{\label{fig6}	
		\includegraphics[width=0.3\textwidth]{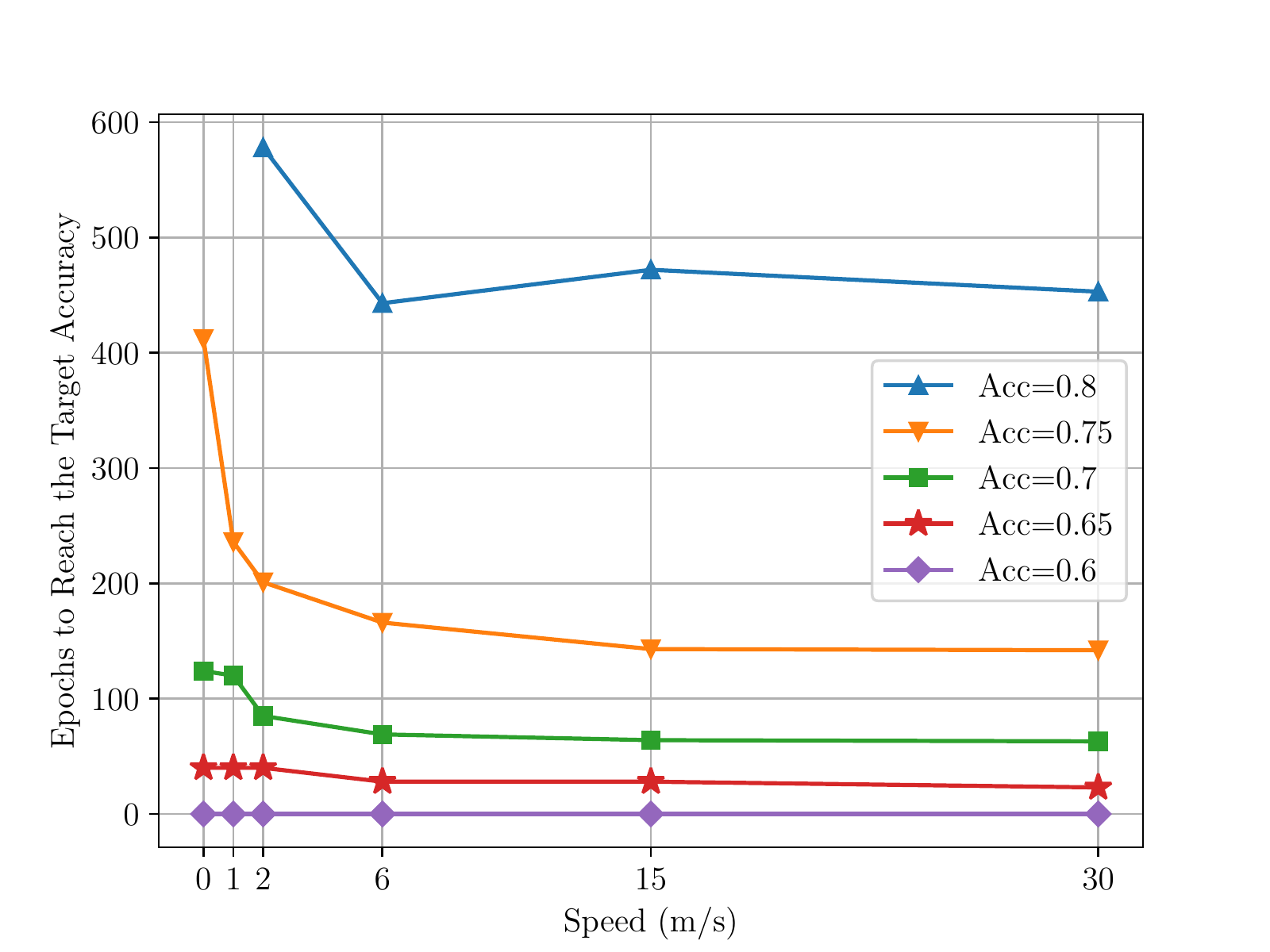} }
    
	\caption{Performance of HFL with different initial data distributions and vehicle speeds.}
	\label{fig5-6}
\end{figure*}

\subsection{Settings}
We consider a city road system on a square grid, where each side of the square is covered by an edge server, e.g., a base station or a road side unit. There are $N=4$ edge servers and a total of $M=32$ vehicles in the system, and each edge server covers 8 vehicles initially. The initial positions of vehicles are randomly distributed. We use the Simulation of Urban MObility (SUMO) simulation platform with the Manhattan model. Each side of the road is $a$ meters long, and vehicles travel on the road at a maximum speed of $v$ meters per second (m/s) and slow down when crossing the intersection. The interval between each edge aggregation is assumed to be $1$ second. We set the default speed to $v=30$. 




We conduct experiments on the CIFAR-10 dataset. To create non-i.i.d. data by partitioning the labels, we choose 40000 training and 8000 test samples in total belonging to 8 classes. For the i.i.d. case, we uniformly divide the training samples into $M$ disjoint subsets, and allocate each subset to one vehicle. For the local non-i.i.d. case, we sort all the data by labels and allocate $l$ classes to each vehicle (denoted by `local non-i.i.d.($l$)' in the following), with all the vehicles holding the same number of samples. For the edge non-i.i.d. case, we first allocate the sorted data to edge servers so that each edge server owns $l$ classes of samples, and then uniformly allocate these samples to the vehicles within their coverage (denoted by `edge non-i.i.d.($l$)' in the following). 

We train a 3-layer CNN with a batch size of 20. The learning rate is set to $\eta=0.1$ for all the vehicles with no momentum and learning rate decay. The local and edge epochs are set to $\tau_e=10$ and $\tau_l=6$. We run the Mob-HierFAVG algorithm for 600 cloud epochs.


\subsection{Performance}
    The performance of HFL under different initial data distributions is shown in Fig \ref{fig2}. To address the impact of mobility, we consider two typical speeds for each distribution, $v=0$ (denoted by `no-mobility') and $v=30$ (denoted by `with-mobility'). For the i.i.d. case and the local non-i.i.d. case, the maximum test accuracies achieved in the no-mobility and with-mobility scenarios are approximately equal. On the other hand, in the edge non-i.i.d. case, the mobility clearly increases the performance. When $l=2$, mobility improves the accuracy by 5.7\% (74.9\% to 80.6\%); and when $l=1$, mobility improves the accuracy by 15.1\% (40.9\% to 56.0\%). These results are aligned with our conjecture in Section \ref{Sec-4}. Furthermore, they show that for the edge non-i.i.d. case, as the number of classes an edge server holds decreases, the accuracy of HFL decreases, while the accuracy improvement brought by mobility increases. 
    
    We then take a further look at the edge non-i.i.d. case. Assume that we start training from a pre-trained model with a test accuracy of 60\%. 
    The results in Fig. \ref{fig5} demonstrate that the test accuracy of the scenarios with mobility, i.e., $v>0$, is obviously higher than the one achieved when $v=0$. Besides, Fig. \ref{fig6} illustrates that a higher speed generally results in faster convergence. In particular, it takes only 142 cloud epochs in the $v=30$ case to reach an accuracy of 75\%, reduced by 65.5\% compared with the $v=0$ case (412 epochs), and 39.8\% compared with the $v=1$ case (236 epochs). These results indicate that mobility can indeed increase the convergence speed and the final test accuracy of HFL. However, the improvement brought by mobility is also limited.
    In our experiments, both the final accuracy and the convergence speed typically saturate beyond $v=6$. This is because the datasets are already sufficiently mixed at this speed, and no further gains can be achieved with even higher mobility.
\section{conclusion}
\label{Sec-6}
In this article, we have investigated the impact of mobility on the convergence of data-heterogeneous HFL. Based on the HFL model with mobility, the convergence analysis of HFL has been conducted, exposing the influence of mobility on the data heterogeneity, and hence, the convergence speed of HFL. Experiments carried out on the CIFAR-10 dataset and the SUMO platform demonstrates the benefits of mobility in the case of edge non-i.i.d. initial data distribution. Specifically, mobility enables an increase in test accuracy by up to 15.1\%. Also, the results show that a higher speed results in faster convergence up to a certain value beyond which the achieved accuracy saturates.

\end{document}